\documentclass[preprint,1p]{elsarticle}
\usepackage{graphicx}
\usepackage[colorlinks,linkcolor=blue,anchorcolor=blue,citecolor=blue]{hyperref}
\usepackage{float}
\usepackage{subfig}
\usepackage{booktabs}
\usepackage{multirow}
\usepackage[normalem]{ulem}
\useunder{\uline}{\ul}{}
\usepackage{caption}
\usepackage{algorithmic}
\usepackage[linesnumbered,ruled]{algorithm2e}
\usepackage{amsmath,amsfonts,amssymb,amsopn,amscd,bbm}
\usepackage{amsthm}
\usepackage{wasysym}

\usepackage{lineno}
\modulolinenumbers[5]
\biboptions{numbers,sort&compress}
\newtheorem{definition}{Definition}
\newtheorem{example}{Example}
\newtheorem{proposition}{Proposition}

\usepackage{url}

\begin{document}
\begin{frontmatter}
    \title{Consistency-guided semi-supervised outlier detection in heterogeneous data using fuzzy rough sets}
    
    \author[inst1]{Baiyang Chen}\ead{farstars@qq.com}
    \author[inst1]{Zhong Yuan\corref{cor1}}\ead{yuanzhong@scu.edu.cn}
    \author[inst1]{Dezhong Peng}\ead{pengdz@scu.edu.cn}
    \author[inst2]{Xiaoliang Chen}\ead{chexiaol@iro.umontreal.ca} %\ead{xdxlchen@gmail.com}
    \author[inst3]{Hongmei Chen}\ead{hmchen@swjtu.edu.cn}
    
    \cortext[cor1]{Corresponding author}
    \affiliation[inst1]{organization={College of Computer Science, Sichuan University},city={Chengdu},postcode={610065},country={China}}
    %\affiliation[inst2]{organization={School of Computer and Software Engineering, Xihua University},city={Chengdu},postcode={610039},country={China}}
    \affiliation[inst2]{organization={Department of Computer Science and Operations Research, University of Montreal},city={Montreal},postcode={QC H3C3J7},country={Canada}}
    \affiliation[inst3]{organization={School of Computing and Artificial Intelligence, Southwest Jiaotong University},city={Chengdu},postcode={611756},country={China}}
    
    \begin{abstract}
        Outlier detection aims to find samples that behave differently from the majority of the data. 
        Semi-supervised detection methods can utilize the supervision of partial labels, thus reducing false positive rates. However, most of the current semi-supervised methods focus on numerical data and neglect the heterogeneity of data information. In this paper, we propose a consistency-guided outlier detection algorithm (COD) for heterogeneous data with the fuzzy rough set theory in a semi-supervised manner. 
        {First, a few labeled outliers are leveraged to construct label-informed fuzzy similarity relations. Next, the consistency of the fuzzy decision system is introduced to evaluate attributes' contributions to knowledge classification. Subsequently, we define the outlier factor based on the fuzzy similarity class and predict outliers by integrating the classification consistency and the outlier factor.}
        The proposed algorithm is extensively evaluated on 15 freshly proposed datasets. Experimental results demonstrate that COD is better than or comparable with the leading outlier detectors.\footnote{This manuscript is the accepted author version of a paper published by Elsevier. The final published version is available at \url{https://doi.org/10.1016/j.asoc.2024.112070}.}
        \end{abstract}
    
    \begin{keyword}
        Semi-supervised outlier detection \sep Fuzzy Rough Sets \sep Heterogeneous data \sep{Label-informed fuzzy similarity relation} \sep Classification consistency
        \end{keyword}
    
    \end{frontmatter}

%\linenumbers

\section{Introduction}

    Outlier detection (OD), also known as anomaly detection or novelty detection, is a process that identifies patterns or data points in a dataset that deviate from what is expected or normal. OD allows for the identification of unusual or unexpected behavior that may have important implications for the system being studied.
    Therefore, it has numerous applications in various domains such as fraud detection \cite{pourhabibi2020fraud}, software defect prediction \cite{jiang2022random}, industry control \cite{wang2019outlier}, medical anomaly detection \cite{Hawkins1981Identification}, etc. 
 
    Since outliers are often rare events and labeled data are frequently insufficient, many OD methods are designed to be unsupervised \cite{Breunig2000LOF, Chen2008Outlier, Jiang2009Some, zhang2009anew, Jiang2010An, Jiang2015Outlier, li2020copod, almardeny2022anovel, li2022robust, yuan2023MFGAD, Su2024GBFRD}. However, in some cases, it may be possible to obtain a limited amount of {labeled outliers to guide the detection process.} Hence, a number of semi-supervised detection algorithms \cite{Pang2018RAMODO, Zhao2018XGBOD, Pang2019DevNet, Pang2023PReNet, Ruff2020DeepSAD, Huang2021ESAD, Zhou2022FEAWAD} have been proposed in the last decades.
    However, most of them are mainly designed to deal with numerical data and neglect the heterogeneity of information. Real-world applications generally include heterogeneous attributes (referred to as mixed attribute data or simply mixed data) where the attributes of objects take various types of values \cite{Zhang2016mixeddata}. For instance, in fraud detection, the data may contain the gender and age of a customer, as well as the date and amount of a transaction. In this case, the attribute gender is nominal (categorical), the age is integer-valued, the transaction date is time-valued, and the transaction amount is real-valued (numerical). Detection of outliers in such a scenario {may be} more practical. 
    However, many detectors inherently regard mixed data as numerical, such that a handful of convenient measures can be {utilized} to extract features or data structures. Unfortunately, this may bring in extra properties that {do not exist.} Taking the attribute gender as an example, one may assign a number 0 to the male, and 1 to the female. Then we would obtain that the female is greater than the male, which does not make any sense and is likely to harm the performance of detection algorithms.

    The fuzzy rough set (FRS) theory \cite{dubois1990rough}, which unifies rough sets and fuzzy sets, is a popular mathematical model for processing data with uncertainty. It has been successfully applied to various domains including feature selection \cite{Zhang2016mixeddata, sang2023active} and outlier detection \cite{Yuan2023WFRDA, Chen2024MFIOD} in mixed data over the last years. 
    In FRS, the concepts of lower and upper approximations from classic Rough Sets are adapted to work within the framework of fuzzy logic, leading to fuzzy lower and upper approximations. This integration offers three advantages: (1) It provides a flexible framework for modeling data with uncertainty and imprecision by representing the boundaries of a concept with fuzzy membership functions. This allows for partial membership where elements can belong to a set to varying degrees, thus more precisely reflecting real-world scenarios. (2) It facilitates the direct processing of various data types, including numerical values, categorical variables, symbols, and more, without the necessity of data type transformation, thereby preserving the data's intrinsic diversity for subsequent analyses. (3) It enhances reasoning and decision-making in ambiguous situations, enabling the classification of objects even when their attribute values do not precisely match the criteria of a specific class, offering a robust approach to dealing with data ambiguity. 
    Therefore, FRS holds significant potential for identifying outliers in heterogeneous data with uncertainty and imprecision.

    The central idea of this paper can be described as follows. 
    Since the goal of outlier detection is to find the minority objects whose behavior is abnormal in data, any object less similar to the others has a higher probability of being an outlier. 
    From the perspective of FRS, the similarity is reflected by the fuzzy similarity class, which can describe how similar an instance is to others. Therefore, this paper employs the fuzzy similarity class to characterize outliers and uses the unlabeled data to construct the outlier factor for each object based on fuzzy similarity classes.
    Then, with a few labeled data, we formulate a fuzzy decision system, where the unlabeled data constitute the conditional attributes (i.e., features) and the class labels form the decision attribute. In this context, given a set of conditional attributes, we can take some measure (e.g., dependency) from the FRS theory to evaluate its classification consistency with the decision. If an attribute set has greater consistency, then it is of high quality for producing a more separable classification.
    {Finally, we combine the outlier factor and the classification consistency to predict outliers in the dataset. }
    
    With the above ideas, we propose a consistency-guided outlier detection algorithm (COD) based on FRS in a semi-supervised manner. The contributions of this paper include:
    \begin{itemize}
        \item {This paper introduces a novel label-informed fuzzy similarity relation for the representation of heterogeneous data. It enables a downstream task-guided approach to determine the optimal fuzzy radius for mixed data from a wide range of applications.}
        \item {We propose to characterize outliers with the fuzzy similarity class and use the classification consistency to guide the scoring of outliers, which improves the accuracy and efficiency of outlier detection with very limited labeled data.}
        \item To our best knowledge, we are the first to propose a novel FRS-based outlier detection model with a semi-supervised approach in mixed data. {It may have the potential to advance the application of the FRS theory in real-world scenarios.}
        \item Extensive experiments on various types of public datasets demonstrate that the proposed algorithm {is better than or comparable with the state-of-the-art methods.}
        \end{itemize}

\section{Related works}
    Depending on the availability of labeled data on which the algorithm relies, outlier detection methods are broadly classified into unsupervised \cite{2020ODsurvey}, semi-supervised \cite{Jiang2023WSAD} and supervised detection algorithms \cite{han2022ADbench}. This paper focuses on semi-supervised outlier detection (SSOD) methods for tabular data that is organized in rows and columns of a table. In particular, we mainly investigate the rank-based, representation learning-based, active learning-based, and reinforcement learning-based detection approaches. It is notable that some methods may belong to more than one group or may combine elements from different groups.

    Rank-based SSOD methods typically leverage a limited number of labeled samples to train a ranking model that scores the outlier degree of all objects.
    For instance, Pang et al. (2019) \cite{Pang2019DevNet} introduce an end-to-end deep framework (DevNet) to learn outlier scores, and incorporate a Gaussian prior and deviation loss to detect outliers.
    Inspired by DevNet, the model PReNet \cite{Pang2023PReNet} ranks the outliers by ordinal regression without involving any assumptions about the probability distribution of outlier scores.
    Besides, Zhou et al. (2022) \cite{Zhou2022FEAWAD} propose a weakly supervised outlier detection model (FEAWAD), which attempts to utilize a deep autoencoder to fit the normal data. The resulting representations are then leveraged to facilitate outlier score learning.
    {Lately, Stradiotti et al. (2024) \cite{Stradiotti2024SSIF} introduce the Semi-Supervised Isolation Forest (SSIF) that integrates labeled and unlabeled data within a probabilistic framework, enhancing performance by leveraging informative split distributions and computing anomaly scores based on labeled instances within tree leaves.}
    
    Representation learning-based detection models can be viewed as an indirect way of learning anomaly scores, and they usually improve unsupervised representation learning by partially available labeled data.
    The earlier solution (e.g., OE \cite{2014OE}, XGBOD \cite{Zhao2018XGBOD}) uses multiple unsupervised detection algorithms to learn useful representations and appends their output anomaly scores to the original features for training a supervised classifier.
    Recent deep learning-based SSOD methods employ end-to-end frameworks to extract anomaly-oriented representations for discovering outliers. 
    One such model is DeepSAD by Ruff et al. (2020) \cite{Ruff2020DeepSAD}. It extends the unsupervised outlier detection algorithm DeepSVDD \cite{Ruff2018DeepSVDD} so that it can make use of labeled data. Moreover, DeepSAD penalizes the inverse of the distances of outliers' representation such that they are projected to further away from the center of a hypersphere where normal data reside.     
    Following DeepSAD, Huang et al. (2021) \cite{Huang2021ESAD} leverage the mutual information between data and their representations as well as an entropy measure, and devise an encoder-decoder-encoder architecture to optimize a KL-divergence-based objective function for semi-supervised outlier detection. 
    The other detection method REPEN \cite{Pang2018RAMODO} uses a ranking-based approach to learn feature representations of ultra-high dimension data and optimizes the detection method based on random distances. A small number of labeled outliers are leveraged to refine the sampling for the training triplets, and help the model learn anomaly-oriented representations.
    {Recently, Zhu et al. (2024) \cite{Zhu2024OSAD} propose Anomaly Heterogeneity Learning (AHL) to simulate diverse heterogeneous anomaly distributions from limited anomaly examples, allowing for detecting both seen and unseen anomalies more effectively than traditional closed-set approaches.}
    Another line of research based on representation learning involves the use of generative adversarial networks (GAN) to improve model training and/or data augmentation.
    Among them, Tian et al. (2022) \cite{Tian2022AA-BiGAN} propose the anomaly-aware bidirectional GAN model (AA-BiGAN) based on the BiGAN \cite{Donahue2016BiGAN} architecture. It uses labeled outliers to learn a probability distribution that is guaranteed to assign low-density values to the collected anomalies.
    Li et al. (2022) \cite{Li2022Dual-MGAN} integrates multiple GANs to realize reference distribution construction and data augmentation for detecting both discrete and grouped anomalies.
    {Liu et al. (2024) \cite{Liu2024Mutual} address challenges in semi-supervised anomaly detection by regularizing the latent representations learned by deep generative models using mutual information maximization, thereby improving separation between normal and abnormal samples.}

    In addition to the above-mentioned approaches, some researchers have adopted active learning and reinforcement learning for SSOD. Active learning-based methods focus on acquiring more informative labeled data with the aid of humans. 
    In one of these works, Gornitz et al.\cite{Gornitz2013SSAD} propose a semi-supervised anomaly detection (SSAD) algorithm that follows the unsupervised learning paradigm, and additionally devises an active learning strategy that simply chooses borderline points for labeling. Another important method that leverages active learning is active anomaly discovery (AAD) \cite{Das2016AAD}. It greedily selects the most likely abnormal samples for labeling and maximizes the number of true outliers under a query budget.
    Reinforcement learning-based SSOD models usually consider the discovery of outliers as a sequential decision process, where human feedback is utilized to help identify more outliers.
    For instance, the model Meta-AAD \cite{Zha2020Meta-AAD} employs deep reinforcement learning to optimize a meta-policy %for selecting samples that are similar to known outliers, 
    to select the most appropriate samples for manual labeling, and optimizes the number of outliers discovered throughout the querying process. Unlike Meta-AAD, Pang et al. \cite{Pang2021DPLAN} try to explore unlabeled data without human help. They design an anomaly-biased simulation environment to enable an RL-based model to discover known and unknown outliers and develop a deep Q-learning-based detection model DPLAN.   
    {Chen et al. (2024) \cite{Chen2024DADS} introduce the Deep Anomaly Detection and Search (DADS) that integrates reinforcement learning to model a Markov decision process for searching possible anomalies in unlabeled data, effectively leveraging both labeled and unlabeled data to enhance detection performance.}

\section{Preliminaries}
    In information processing systems, data is usually stored in a table (also called an information system, information table, etc.), where every row corresponds to an object (sample), and every column denotes an attribute (feature). {This section reviews some fundamental concepts of FRS that help understand the subsequent sections of this paper.}
    
    \begin{definition}
        An information system is a tuple $(U, A)$, where $U=\{x_1,x_2,\ldots,x_n\}$ is the set of objects, also referred to as the universe of discourse, and $A=\{a_1,a_2,\ldots, a_m\}$ is the set of attributes that every object has. 
	  \end{definition}
     
   \begin{definition}
        Given an information system $(U, A)$. If $\widetilde{X}$ is a map from $U$ to $[0,1]$, then $\widetilde{X}$ is a fuzzy set on $U$, i.e. $\widetilde{X}:U\to [0,1]$.
	\end{definition}

    $\forall x_i \in U$, $\widetilde{X}(x_i)$ is the membership of $x_i$ to $\widetilde{X}$, or the membership function of $\widetilde{X}$. The fuzzy set is often denoted by $\widetilde{X}=\left(\widetilde{X}({{x}_{1}}),\widetilde{X}({{x}_{2}}),\ldots ,\widetilde{X}({{x}_{n}})\right)$, and the fuzzy cardinality of  $\widetilde{X}$ is computed by $|\widetilde{X}|=\sum\limits_{i}{{\widetilde{X}}({x}_{i})}$.
 
    \begin{definition}
        Let $U$ be a set of objects, a fuzzy relation $\widetilde{R}$ on $U$ is defined as {a family of fuzzy sets} $\widetilde{R}: U\times U \rightarrow [0,1]$.
        \end{definition}
     
    Some commonly employed operations of fuzzy relations are listed as follows.
    \begin{enumerate}[\indent(1)]
        \item $\widetilde{R}_1=\widetilde{R}_2 \Leftrightarrow \forall (x_i, x_j)\in U\times U, \widetilde{R}_1(x_i, x_j)=\widetilde{R}_2(x_i, x_j)$;
        \item $\widetilde{R}_1\subseteq \widetilde{R}_2 \Leftrightarrow \forall (x_i, x_j)\in U\times U, \widetilde{R}_1(x_i, x_j)\le \widetilde{R}_2(x_i, x_j)$;
        \item $(\widetilde{R}_1\cup\widetilde{R}_2)(x_i, x_j)=\max\left\{\widetilde{R}_1(x_i, x_j), \widetilde{R}_2(x_i, x_j)\right\}$;
        \item ${(\widetilde{R}_1\cap \widetilde{R}_2)}(x_i, x_j)=\min\left\{\widetilde{R}_1(x_i, x_j),\widetilde{R}_2(x_i, x_j)\right\}$.
	\end{enumerate}
    
    $\forall (x_i,x_j)\in U\times U$, the membership $\widetilde{R}(x_i,x_j)$ expresses the degree to which $x_i$ has a relation $\widetilde{R}$ with $x_j$. A fuzzy relation $\widetilde{R}$ on $U$ is usually denoted by a fuzzy relation matrix $M({\widetilde{R}})=(r_{ ij})_{n\times n}$, where $r_{ij}=\widetilde{R}(x_i,x_j)$.
    $\forall x_1, x_2, x_3 \in U$, if a fuzzy relation $\widetilde{R}$ meets: (1) reflexive: $\widetilde{R}(x_1, x_1)=1$, (2) symmetric: $\widetilde{R}(x_1, x_2)=\widetilde{R}(x_2, x_1)$, then $\widetilde{R}$ is also called a fuzzy similatrity relation. %If $\widetilde{R}$ further meets: (3) transitive: $\widetilde{R}(x_1, x_2)\ge \min(\widetilde{R}(x_1, x_3),\widetilde{R}(x_3, x_2))$,  then $\widetilde{R}$ is also called a fuzzy equivalence relation. 

    In Fuzzy Rough Sets, the tuple $(U,\widetilde{R})$ is called a fuzzy approximation space \cite{dubois1990rough}, where the approximation of a fuzzy set can be defined.

    \begin{definition}\label{def:low_appr}
        Let $(U,\widetilde{R})$ be a fuzzy approximation space, and $\widetilde{X}$ is a fuzzy set on $U$,  the lower approximation of $\widetilde{X}$ is a fuzzy set. Its membership function is defined as
	\begin{equation}\label{eq_low_appr}
            \underline{\widetilde{R}}\widetilde{X}(x_i)=\underset{x_j\in U}{\mathop{\inf }}\,\max \bigg\{1-\widetilde{R}(x_i,x_j), \widetilde{X}(x_j)\bigg\}.
  		\end{equation}
        \end{definition}
    
    The fuzzy lower approximations $\underline{\widetilde{R}}\widetilde{X}$ of $\widetilde{X}$ are used to express the degrees of elements certainly belonging to $\widetilde{X}$.
    %, and it is also referred to as the fuzzy positive region.
    %It is easy to see that the fuzzy approximation space can degrade to the corresponding Pawlak approximation space when the equivalence relation and the object set to be approximated are both crisp. 
    Specifically, when the fuzzy set to be approximated is crisp, Equation~(\ref{eq_low_appr}) can be rewritten as
    \begin{equation}\label{eq_low_appr_crisp}
        \underline{\widetilde{R}}X(x_i)=\underset{x_j\notin X}{\mathop{\inf }}\,\{1-\widetilde{R}(x_i,x_j)\}.
        \end{equation}

    Given a fuzzy approximation space $(U, \widetilde{R})$, the fuzzy similarity relation $\widetilde{R}$ can induce a generalized fuzzy partition of $U$, i.e., a set of fuzzy similarity classes which are constructed by collecting a group of fuzzy targets. 
    \begin{definition}\cite{yuan2022FRGOD}\label{def:partition}
        {The generalized fuzzy partition of $U$ induced} by a fuzzy similarity relation $\widetilde{R}$ is defined as
        \begin{equation}
            U/{\widetilde{R}}=\left\{[x_i]_{\widetilde{R}}\right\}_{x_i\in U},
            \end{equation}
        where $[x_i]_{\widetilde{R}}=\left(r_{i1}^{\widetilde{R}}, r_{i2}^{\widetilde{R}}, \dots, r_{in}^{\widetilde{R}}\right)$ is a fuzzy similarity class containing $x_i$. 
        \end{definition}
    {Each} fuzzy similarity class $[x_i]_{\widetilde{R}}$ is a fuzzy set, which clearly describes how similar the object $x_i$ is to all objects in $U$. The family of $U/\widetilde{R}$ forms a fuzzy concept system on $U$ \cite{hu2006fuzzy}.
    Obviously, $[x_i]_{\widetilde{R}}(x_j)=\widetilde{R}(x_i,x_j)=r^{\widetilde{R}}_{ij}$. If $\widetilde{R}(x_i,x_j)=1$, then it suggests that $x_j$ certainly belongs to $[x_i]_{\widetilde{R}}$; If $\widetilde{R}(x_i,x_j)=0$, then $x_j$ definitely does not belong to $[x_i]_{\widetilde{R}}$. The fuzzy cardinality of $[x_i]_{\widetilde{R}}$ is calculated by
    \begin{equation}\label{eq_cardinality}
        \big|[x_i]_{\widetilde{R}}\big|=\sum\limits_{j=1}^{n}{{\widetilde{R}}({{x}_{i}},{{x}_{j}})}.
        \end{equation}
    We can easily obtain $1\le \big|[x_i]_{\widetilde{R}}\big|\le n$. {The cardinality of $[x_i]_{\widetilde{R}}$ reflects the overall similarity of the object $x_i$ to others in $U$ based on the knowledge $\widetilde{R}$.}

    In the application of Fuzzy Rough Sets, the information system $(U, A)$ is expressed as a fuzzy information system. If $A=C \cup d$ and $C\cap d=\emptyset$, the fuzzy information system is also referred to as a fuzzy decision system defined as
    \begin{definition}\label{def:FDS}
        A fuzzy decision system is a tuple $(U, C\cup d)$, where $U=\{x_1,x_2,\ldots,x_n\}$ is the set of objects, $C$ represents the conditional attributes and $d$ denotes the decision attribute.
        \end{definition}

    \begin{example}\label{ex:low_appr}
        Let $U=\{x_1,x_2,x_3\}$ be the set of objects, $A=\{a_1,a_2\}$ is the set of attributes, and the fuzzy similarity relation matrices on $U$ induced from the attribute $a_1$ and $a_2$ respectively are
        
        $M({\widetilde{a}_1})=\left(\begin{array}{ccc}
            1 & 0 & 0.9\\
            0 & 1 & 0\\
            0.9 & 0 & 1\\
            \end{array}\right)$, 
        $M({\widetilde{a}_2})=\left(\begin{array}{ccc}
            1 & 0.8 & 0\\
            0.8 & 1 & 0.9\\
            0 & 0.9 & 1\\
            \end{array}\right)$.\\
        Then, the fuzzy similarity class $[x_1]_{\widetilde{a}_1}$ generated by $\widetilde{a}_1$ containing $x_1$ is $(1,0,0.9)$, and $\big|[x_1]_{\widetilde{a}_1}\big|=1.9$. Similarly, $[x_2]_{\widetilde{a}_1}=(0,1,0)$, $\big|[x_2]_{\widetilde{a}_1}\big|=1$, $[x_3]_{\widetilde{a}_1}=(0
        .9,0,1)$, $\big|[x_3]_{\widetilde{a}_1}\big|=1.9$.
        
        Let $\widetilde{X}=(0.2,0.5,0.8)$ be a fuzzy set on $U$, then the fuzzy lower approximation $\underline{\widetilde{a}_1}\widetilde{X}$ of $\widetilde{X}$ with regard to ${\widetilde{a}_1}$ can be computed using Eq.~(\ref{eq_low_appr}) as
        
        $\max\big(1-\widetilde{a}_1(x_1,x_j),\widetilde{X}(x_j)\big)\big|_{x_j=x_1}=\max(0,0.2)=0.2$,
        
        $\max\big(1-\widetilde{a}_1(x_1,x_j),\widetilde{X}(x_j)\big)\big|_{x_j=x_2}=\max(1,0.5)=1$,
        
        $\max\big(1-\widetilde{a}_1(x_1,x_j),\widetilde{X}(x_j)\big)\big|_{x_j=x_3}=\max(0.1,0.8)=0.8$.\\
        Then we have $\underline{\widetilde{a}_1}\widetilde{X}(x_1)=\underset{x_j\in U}{\mathop{\inf }}\,\max\big(1-\widetilde{a}_1(x_1,x_j),\widetilde{X}(x_j)\big)=0.2$.
        Similarly, we can obtain
        
        $\underline{\widetilde{a}_1}\widetilde{X}(x_2)=0.5$, $\underline{\widetilde{a}_1}\widetilde{X}(x_3)=0.2$. Therefore, $\underline{\widetilde{a}_1}\widetilde{X}=(0.2, 0.5, 0.2)$.
        \end{example}

    {Having illustrated the basic calculations in FRS through the preceding example, we will now proceed to an in-depth exploration of our methods.}
        
\section{Methodology}
    {This section presents the consistency-guided outlier detection (COD) method in detail. First, we leverage a few labeled outliers to construct label-informed fuzzy similarity relations, which enable a flexible and adaptable representation of heterogeneous data. Next, the consistency of the fuzzy decision system is introduced to assess attributes' contributions to knowledge classification. Subsequently, we define the outlier factor based on the fuzzy similarity class and predict outlier scores by integrating the classification consistency and the outlier factor. Finally, a threshold is determined for binary outlier predictions. The overall framework of COD is illustrated in Figure \ref{fig_COD}.}
    
    \begin{figure*}[!h]
        \centering
        \includegraphics[width=\textwidth]{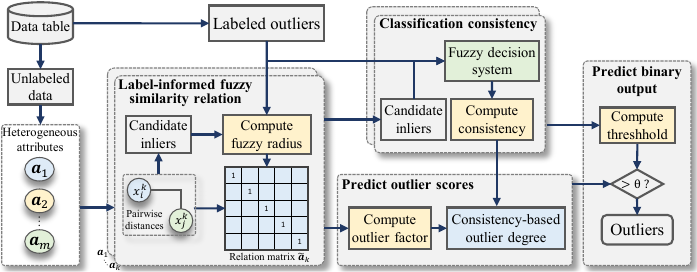}
        \caption{Overall framework of COD.}
        \label{fig_COD}
        \end{figure*}
        
\subsection{Label-informed fuzzy similarity relation} % fuzzy similarity relation
    {This part begins with data preprocessing, and then constructs the fuzzy similarity relation for each attribute with the guidance of labeled data, which is termed the label-informed fuzzy similarity relation.}
    
    As real-world datasets are usually heterogeneous, this paper recognizes two broad types of data, i.e., the numerical data including real-valued and integer-valued numbers, and the nominal data including categorical and symbolic data. Since the magnitude and dimension of raw numerical values are diverse, the numerical data are first normalized into the interval of $[0, 1]$ by the min-max normalization {during data preprocessing}. 
    
    In order to jointly represent both numerical and nominal values, we adopt a {hybrid fuzzy membership function \cite{yuan2022FRGOD} to construct} the fuzzy similarity relation between each pair of objects in the datasets. 
    {Let $f_i^k$ be the value of attribute $a_k$ for object $x_i$,
    and $d_{ij}^k=\left|f_i^k - f_j^k\right|$} is the {distance} between $x_i$ and $x_j$ on the attribute $a_k$, the fuzzy similarity relation $\widetilde{a}_k(x_i,x_j)$ between $x_i$ and $x_j$ induced by the attribute $a_k$ is calculated by
    \begin{equation}\label{eq_rel_ak}
        \widetilde{a}_k(x_i,x_j)=\left\{
        \begin{array}{ll}
        1,          & \text{if } f_i^k=f_j^k, a_k \text{ is nominal};\\
        0,          & \text{if } f_i^k\neq f_j^k, a_k \text{ is nominal};\\
        1-d_{ij}^k, & \text{if } d_{ij}^k\le \lambda_k, a_k \text{ is numerical};\\
        0,          & \text{if } d_{ij}^k > \lambda_k, a_k \text{ is numerical};\\
        \end{array}\right.
        \end{equation}
    {where $\lambda_k\in [0,1]$ is the fuzzy radius, which is usually determined empirically \cite{Yuan2023WFRDA} or by some heuristic rules (e.g., coverage and specificity \cite{Pedrycz2016Design}) in unsupervised scenarios. However, in the label-informed setting, we are able to design a downstream task-guided approach to determine their values with a few labeled objects. The optimal $\lambda_k$ is derived through the maximization of the following equation.
    \begin{equation}\label{eq_lamb}
        \lambda_k=\arg\max_{\lambda} \frac{\sum_{x_i\in U^-}\left|[x_i]_{\widetilde{a}_k}^\lambda\right|}{|U^-|\cdot |U|} - \frac{\sum_{x_j\in U^+}\left|[x_j]_{\widetilde{a}_k}^\lambda\right|}{|U^+|\cdot |U|},
        \end{equation}
    where $U^-$ and $U^+$ denote the subset of the negative instances and the positive instances in $U$, respectively. The cardinality of the fuzzy similarity classes $[x_i]$ reflects the overall similarity of a negative object $x_i$ to all others in $U$ based on the knowledge ${\widetilde{a}_k}$, and likewise, the cardinality of $[x_j]$ reflects that of a positive object $x_j$. The objective is to find an optimal value of $\lambda$ by balancing the average cardinality of labeled outliers and normal points. Given the assumption that outliers are less similar to others than the majority of data, the objects in $U^-$ exhibit higher similarity to each other. Consequently, the cardinality $|[x_i]_{\widetilde{a}_k}^\lambda\|$ should be greater. Conversely, outliers in $U^+$ tend to be less similar to the majority of $U$, resulting in smaller cardinality $|[x_j]_{\widetilde{a}_k}^\lambda|$. %The optimal $\lambda_k$ for each attribute $a_k$ can be obtained by Dynamic programming or Generic Algorithms.

    The label-informed fuzzy similarity relation allows for modeling heterogeneous data unitedly and efficiently. However, in case normal objects in $U^-$ are not explicitly identified, it is necessary to designate certain instances as candidate negative instances. Previous works \cite{Pang2023PReNet, Zhou2022FEAWAD, Pang2018RAMODO} simply treat unlabeled data as inliers under the assumption that outliers are often rare events. However, this approach may not work well in cases where the outlier contamination level is high (i.e., the proportion of outliers in the unlabeled data), potentially introducing noise. 
    This paper introduces a more practical solution for selecting candidate inliers.  
    
    Let $d_{ij}^k$ be the pairwise distance between $x_i$ and $x_j$ with respect to the attribute $a_k$. Then the average distance of $x_i$ to others in $U$ is $D_i^k=\frac{\sum_{x_j\in U}d_{ij}^k}{|U|}$. We first sort the objects in {ascending} order based on their average distances as {$D_{i_1}^k \leq D_{i_2}^k \leq \ldots \leq D_{i_{|U|}}^k$}.
    Then, the set of candidate inliers is defined as
    \begin{equation}\label{eq_n_neg}
        U^- = \left\{x_{i_k} \mid k = 1, 2, \ldots, N_-\right\}.
        \end{equation}
    
    This selection process aims to identify the top $N_-$ objects with the greatest average distance, reflecting their greater similarity to others, and designates them as candidate inliers.

    Given an attribute subset $B\subseteq C$, the fuzzy similarity relation $\widetilde{B}$ on $U$ can be derived from Equation~(\ref{eq_rel_ak}) using conjunction operation \cite{Yuan2021ARsurvey} as
    \begin{equation}\label{eq_rel_B}
        \widetilde{B}(x_i, x_j)=\min_{a_k\in B}{{\widetilde{a}_k}(x_i, x_j)}. 
    	\end{equation}
    \begin{proposition}\label{prop_subset}
        For any attribute subset $ B, P\subseteq C$, if $B\subseteq P$, then $\widetilde{P} \subseteq \widetilde{B}$.
    \end{proposition}
    Proposition \ref{prop_subset} expresses the inclusion relation between fuzzy relations induced by attributes: the more attributes adopted, the smaller (fine-grained) the fuzzy relation. Notably, given any $x_i$ and $x_j$ in $U$, if the fuzzy relation $\widetilde{P}$ is sufficiently fine-grained, the degree to which $x_i$ has a relation $\widetilde{P}$ with $x_j$ approaches 0. In other words, every object in $U$ will be distinct from each other under the knowledge of $\widetilde{P}$.}

\subsection{Classification consistency}
    {
    In a fuzzy decision system $(U, C\cup d)$, both the condition attributes in $C$ and the decision attribute $d$ can induce a fuzzy partition, which inherently indicates a knowledge classification. To evaluate attributes’ contributions to knowledge classification, we introduce classification consistency based on the notion of fuzzy dependency in the following.

    \begin{definition}\cite{Jensen2007fuzzy} \label{def:dependency}
        The fuzzy dependency $\gamma_{C}(d)$ of the decision attribute $d$ on the conditional attributes $C$ is
    \begin{equation}\label{eq_dependency}
        \gamma_{C}(d)=\frac{\left|\bigcup_{\widetilde{X}\in U/\widetilde{d}}\underline{\widetilde{C}}\widetilde{X}\right|}{|U|}.
        \end{equation}
        \end{definition}}
    It is obvious that $0<\gamma_{C}(d)\leq1$, and $\gamma_{C}(d)=1$ means that the decision $d$ can be approximated accurately by the attributes  $C$. 
    In the application of outlier detection, all objects in $U$ are partitioned into two imbalanced crisp classes by the decision attribute $d$ as 
    \begin{equation}\label{eq_partition_d}
            U/d =\left\{[x^-]_{d}, [x^+]_{d} \right\},
            \end{equation}
    where $[x^-]_{d}$ and $[x^+]_{d}$ represent the two crisp equivalence classes i.e., the negative class and the positive class, respectively. In this context, given an conditional attribute subset $B\subseteq C$, the fuzzy dependency between $B$ and $d$ is computed by
        \begin{equation}\label{eq_gamma}
            \gamma_{B}(d) =\frac{\left|POS_{B}(d) \right|}{\left|U\right|}
            =\frac{\left|\underline{\widetilde{B}}[x^-]_{d}\cup \underline{\widetilde{B}}[x^+]_{d}\right|}{\left|U\right|}.
            \end{equation}
    %As follows, we give the definition of consistency for outlier detection.

    To address the imbalance problem in the task of outlier detection, we define the consistency based on the fuzzy dependency {through balancing the weights between the two classes}.

\begin{definition}\label{def:consistency}
    Given an attribute subset $B \subseteq C$, the consistency $\xi_{B}(d)$ of the decision attribute $d$ on the conditional attributes $B$ is defined as
    \begin{equation}\label{eq_consistency}
    \xi_{B}(d) =\frac{\left|\underline{\widetilde{B}}[x^-]_{d}\right|}
    {\left|U^-\right|} + \frac{\left|\underline{\widetilde{B}}[x^+]_{d}\right|} {\left|U^+\right|}.
    \end{equation}
    \end{definition}

    {
    \begin{proposition}\label{prop_dependency}
        $\forall B, P\subseteq C$, if $B\subseteq P$, then $\xi_{B}(d)\leq \xi_{P}(d)$.
    \end{proposition}
    \begin{proof}
        Given $B\subseteq P$, according to Proposition \ref{prop_subset},  we have $\widetilde{B} \supseteq \widetilde{P}$. Therefore, $\forall{x_i, x_j\in U}$, we can obtain $1-\widetilde{B}(x_i,x_j)\leq 1-\widetilde{P}(x_i,x_j)$ . Let $[x^-]_{d}$ and $[x^+]_{d}$ be the negative class and the positive class of $U$, by Definition \ref{def:low_appr} and Eq.~(\ref{eq_low_appr_crisp}), we have $\underline{\widetilde{B}}[x^-]_{d}\subseteq \underline{\widetilde{P}}[x^-]_{d}$, and  $\underline{\widetilde{B}}[x^+]_{d}\subseteq \underline{\widetilde{P}}[x^+]_{d}$. Further,  $\frac{\left|\underline{\widetilde{B}}[x^-]_{d}\right|}{|U^-|}\leq \frac{\left|\underline{\widetilde{P}}[x^-]_{d}\right|}{|U^-|}$, and  $\frac{\left|\underline{\widetilde{B}}[x^+]_{d}\right|}{|U^+|}\leq \frac{\left|\underline{\widetilde{P}}[x^+]_{d}\right|}{|U^+|}$. Therefore,  $\xi_{B}(d)\leq \xi_{P}(d)$.
    \end{proof}
    The above proposition suggests that the more conditional attributes employed, the higher the degree of classification consistency of the fuzzy decision system will be.}

\subsection{Consistency-guided outlier degree}
    In a fuzzy decision system, a fuzzy similarity relation $\widetilde{B}$ can induce a fuzzy partition of $U$, i.e., a set of fuzzy similarity classes. Each fuzzy similarity class $[x_i]_{\widetilde{B}}$ is a fuzzy set, which clearly describes how similar the object $x_i$ is to all objects in $U$. As the aim of outlier detection is to find the minority objects whose behavior is abnormal, any object which is relatively less similar to the other objects, has a higher probability of being an outlier. Therefore, we {define the outlier factor based on the fuzzy similarity class.}     
    \begin{definition}\label{def:OF}
         Let $(U,C\cup d)$ be a fuzzy decision system, $\forall {B}\in C$, the outlier factor generated by the attribute subset $B$ is defined as
        \begin{equation}\label{eq_OF}
            OF_B(x_i)=1-\frac{1}{|U|}\left|[x_i]_{\widetilde{B}}\right|.
            \end{equation}
        \end{definition}
        
    In the above definition, {choosing an appropriate attribute set ${B}$ is crucial for constructing an outlier factor. To achieve this, it is essential to have a quality metric for candidate attributes, and one such measure is the classification consistency defined in the previous section.
    If an attribute subset has greater classification consistency, then it is of higher quality for producing a more separable partition.}    
    Ideally, all the $2^m$ subsets of conditional attributes should be considered for constructing outlier factors. But this procedure is exhaustive and may be too costly and practically prohibitive even for a medium-sized $m$. 
    {Moreover, as reflected by Proposition \ref{prop_subset}, a fine-grained fuzzy relation tends to treat each object as a distinct category, which does not help to distinguish which objects are anomalous. Therefore, following the previous works \cite{li2022ecod, Yuan2023WFRDA, Liu2023FGAS}, this paper also adopts a straightforward solution in that each singleton attribute is taken to construct a similarity class.}

    On this basis, {we can predict the outlier score for each instance by integrating the $OF$s of every attribute and its corresponding consistency} in the form of a weighted summation. 
    \begin{definition}\label{def:OD}
    
    Let $(U,C\cup d)$ be a fuzzy decision system, $\forall x_i\in U$, the consistency-guided outlier degree \text{COD} of $x_i$ is defined as
    \begin{equation}\label{eq_COD}
        \text{COD}(x_i)=\frac{1}{|C|}\sum\limits_{a_k\in C}{OF_{a_k}(x_i) \cdot {\xi}_{{a}_k}(d)}.
        \end{equation}
        \end{definition}

    {To illustrate the computation process of the aforementioned method more clearly, a concrete example is provided in the following.}

\begin{table}[!h]
    \centering
    
    \caption{A data table for patients.}\label{tab:HIS}
    \tabcolsep=6pt
    
    \begin{tabular}{c|cccc|cccc}
    \toprule
    $U$    & $a_1$   & $a_2$  & $a_3$ & $d$      & $a_1$   & $a_2$  & $a_3$ & $d$     \\ \midrule
    $x_1$  & \male   & 38   &  62.5   & Negative & \male   & 0      & 0.503 & Negative\\
    $x_2$  & \male   & 47   &  72.3   & Negative & \male   & 0.692  & 1     & Negative\\
    $x_3$  & \female & 51   &  52.6   & Positive & \female & 1      & 0     & Positive\\
    $x_4$  & \female & 44   &  65.6   & Negative & \female & 0.462  & 0.66  & Negative\\ \bottomrule

    \end{tabular}
    \end{table}
    
\begin{example}
    
    Let $(U, C\cup d)$ be a data table for patients with heterogeneous attributes (as shown in the left of Table~\ref{tab:HIS}), and $U=\{x_1, x_2, x_3, x_4\}$ denotes the persons, $C=\{a_1,a_2,a_3\}$ represents the set of conditional attributes for each patient, where $a_1$, $a_2$ and $a_3$ represent the \textit{gender}, the \textit{age} and the \textit{weight}, respectively. $d$ is the decision attribute indicating whether the person has been diagnosed with a disease.

    We first transform the numerical attributes $a_2$ and $a_3$ into the same magnitude by min-max normalization, the results are shown in the right of Table~\ref{tab:HIS}. Then we construct the label-informed fuzzy similarity relations in the following. 
    
    From the decision attribute $d$, we have the set of positive instances $U^+=\{x_3\}$ and the set of negative instances $U^-=\{x_1,x_2,x_4\}$.
    For the categorical attribute $a_1$, we can easily obtain the fuzzy similarity relation matrix by Eq.~(\ref{eq_rel_ak}) as
    
    $M(\widetilde{a}_1)=\left(\begin{array}{cccc}
        1 & 1 & 0 & 0\\
        1 & 1 & 0 & 0\\
        0 & 0 & 1 & 1\\
        0 & 0 & 1 & 1\\
        \end{array}\right)$.\\
    For the numerical attributes, we first calculate the fuzzy radius by Eq.~(\ref{eq_lamb}) as
    $\lambda_2\approx 0.539$ for $a_2$ and $\lambda_3\approx 0.498$ for $a_3$ .
    Then, we have the fuzzy similarity relation matrix by Eq.~(\ref{eq_rel_ak}) as
    
    $M(\widetilde{a}_2)=\left(\begin{array}{cccc}
        1	 & 0    & 0	    & 0.539\\
        0    & 1	& 0.692	& 0.769\\
        0	 & 0.692& 1 	& 0.462\\
        0.539& 0.769& 0.462	& 1    \\
        \end{array}\right)$, 
    $M(\widetilde{a}_3)=\left(\begin{array}{cccc}
        1	 &0.503	&0	&0.843 \\
        0.503&	1	&0	&0.660 \\
        0	 &0	    &1	&0     \\
        0.843&0.660	&0	&1     \\
        \end{array}\right)$.
        
    Next, we compute the classification consistency for each attribute. From Table~\ref{tab:HIS}, we have the positive and negative class induced by $d$ as $[x^+]_{d}=(0,0,1,0)$, $[x^-]_{d}=(1,1,0,1)$. Then we obtain the following low approximations (refer to Example \ref{ex:low_appr}):
        
    $\underline{\widetilde{a}_1}[x^+]_{d}=(0,0,0,0)$,
    $\underline{\widetilde{a}_1}[x^-]_{d}=(1,1,0,0)$,
    
    $\underline{\widetilde{a}_2}[x^+]_{d}=(0,0,0.308,0)$,
    $\underline{\widetilde{a}_2}[x^-]_{d}=(1,0.308,0,0.539)$,
    
    $\underline{\widetilde{a}_3}[x^+]_{d}=(0,0,1,0)$,
    $\underline{\widetilde{a}_3}[x^-]_{d}=(1,1,0,1)$.\\
    By Definition \ref{def:consistency}, we have
    $\xi_{a_1}(d)=\frac{2}{3}+\frac{0}{1}\approx0.667$, 
    $\xi_{a_2}(d)=\frac{1.846}{3}+\frac{0.308}{1}\approx0.923$, 
    $\xi_{a_3}(d)=\frac{3}{3}+\frac{1}{1}=2$.
    
    In the following, we calculate the outlier factor for each instance. By Definition \ref{def:OF}, we have
    
    $OF_{a_1}(x_1)=1-\frac{1}{|U|}|[x_1]_{\widetilde{a_1}}|=1-\frac{1}{4}\times2=0.5$,
    $OF_{a_2}(x_1)=1-\frac{1}{4}\times1.539\approx0.615$,
    
    $OF_{a_3}(x_1)=1-\frac{1}{4}\times2.346\approx0.414$.\\
    Similarly, we can compute the outlier factor for the other instances.
    
    Finally, the outlier degree of each instance is obtained by integrating the OFs of all attributes and their corresponding consistencies by Definition \ref{def:OD}.
    
    $COD(x_1)=\frac{1}{|C|}\left(OF_{a_1}(x_1)\times\xi_{a_1}(d)+OF_{a_2}(x_1)\times\xi_{a_2}(d)+OF_{a_3}(x_1)\times\xi_{a_3}(d)\right)\\
    \text{\hspace{1.85cm}}=\frac{1}{3}(0.5\times0.667+0.615\times0.923+0.414\times2)\approx0.576$.\\
    Similarly, we have $COD(x_2)\approx0.536$,$COD(x_3)\approx0.753$,$COD(x_4)\approx0.455$.
    \end{example}

\subsection{Binary outlier classification}
    {After associating each object with an outlier score, a binary output is generated to indicate whether an input instance is an outlier or not. This is achieved by determining a threshold $\theta$ as follows.
    
    \begin{definition}\label{def:COD}
        Let $\theta$ be a real-valued threshold, $\forall x_i \in U$, if the outlier score \textit{COD}$(x_i) > \theta$, then $x_i$ is regarded as a consistency-guided outlier.
        \end{definition}

    In the label-informed scenario, where labeled outliers are available, the optimal threshold value is obtained adaptively. A straightforward approach is to utilize the smallest outlier degree of the labeled outliers. Let $U^+$ be the set of labeled outliers, and the optimal threshold ${\theta}^*$ can be determined as the minimum outlier degree of the labeled outliers:
    \begin{equation}\label{eq_theta}
        {\theta}^*=\min_{x_i\in U^+} {COD}(x_i)
    \end{equation}
    }

    The whole procedure of the proposed outlier detection algorithm COD is illustrated in Algorithm \ref{alg:COD}. COD begins by finding an optimal fuzzy radius for each attribute. It then constructs label-informed fuzzy similarity relations and calculates the classification consistency for each attribute. Subsequently, COD computes the outlier factor and the outlier degree for each input instance. Finally, a threshold value is computed, which is used to predict a binary output designating outliers. {Since COD's primary operation involves computing the distance between each pair of instances for each attribute, the worst time complexity for COD is $O(mn^2)$, where $m$ represents the number of attributes and $n$ represents the number of instances in the dataset.}

    \begin{algorithm}[!h]
    \caption{Consistency-guided outlier detection}\label{alg:COD}
    \LinesNumbered
    \KwIn{A set of $n$ objects $U$ with $m$ conditional attributes $C$, a subset of labeled {outliers $U^{+}\subseteq U$}.}
    \KwOut{Outlier set $O$}
    $O \leftarrow \emptyset$\;
    {
    \For{Each $a_k\in C$}{
        % Select candidate negative objects $U^-$ by Eq.(\ref{eq_n_neg})\;
        Compute fuzzy radius $\lambda_k$ by Eq.(\ref{eq_lamb})\;
        Construct fuzzy relation matrix $M({a_k})$ using Eq.~(\ref{eq_rel_ak})\;
        % Select negative objects $U^-$ by Eq.(\ref{eq_n_neg})\;
        % Compute $|[x_i]_{\widetilde{a}_k}|$ for each $x_i \in U^-\cup U^+$ using Eq.~(\ref{eq_cardinality})\;
        Compute the consistency $\xi_{a_k}(d)$ by Definition~(\ref{def:consistency})\;
        }
    }
    \For{Each $x_i\in U$}{
        \For{Each $a_k\in C$}{
        Compute outlier factor $OF_{a_k}(x_i)$ using Eq.~(\ref{eq_OF})\;
        }
        Compute outlier degree $COD(x_i)$ using Eq.~(\ref{eq_COD})\;
        }

    {Compute the threshold $\theta$ by Eq.(\ref{eq_theta})}\;
    \For{Each $x_i\in U$}{
        \If {$COD(x_i)>\theta$}{
        $O \leftarrow O \cup \{x_i\}$\;
        }
    }
    \Return $O$.
    \end{algorithm}

\section{Experiments}
    This section conducts comparison experiments with some state-of-the-art methods to evaluate our proposed algorithm on {20 public datasets, which range over various fields including images, medics, biologics, etc.} All the datasets and source codes are publicly available online\footnote{https://github.com/ChenBaiyang/COD}.

\subsection{Datasets}
    {The experimental datasets include 4 categorical, 5 mixed}, and 11 numerical datasets. The number of samples in a dataset {ranges from 226} to 11183, and the ratio of anomalies varies between {0.8\%} and 35.9\%. The details of the datasets are provided in Table \ref{tab:datasets}. 
    
    \begin{table*}[!h]
        \centering
        \caption{The statistics of the datasets used in the experiments.}\label{tab:datasets}
        \resizebox{0.95\textwidth}{!}{
        \begin{tabular}{cccccccc}
        \toprule
        {No.} & Datasets    & \# Samples & \# Attributes & \# Outlier & \% Outlier & Category & Data type\\
        \midrule
        1  & Annealing   & 798   & 38  & 42  & 5.3\%  & Physical    & Mixed     \\
        2  & {Annthyroid}  & 7200  & 6   & 534 & 7.4\%  & Healthcare  & Numerical \\
        3  & Arrhythmia  & 452   & 279 & 66  & 14.6\% & Medical     & Mixed     \\
        4  & Audiology   & 226   & 69  & 57  & 25.2\% & Healthcare  & Categorical\\
        5  & Breast      & 286   & 9   & 85  & 29.7\% & Medical     & Categorical\\
        6  & {Breastw}     & 683   & 9   & 239 & 35.0\% & Healthcare  & Numerical \\
        7  & {Cardio}      & 1831  & 21  & 176 & 9.6\%  & Healthcare  & Numerical \\
        8  & CreditA     & 425   & 15  & 42  & 9.9\%  & Commercial  & Mixed     \\
        9  & Ionosphere  & 351   & 32  & 126 & 35.9\% & Oryctognosy & Numerical \\
        10 & {Mammography} & 11183 & 6   & 260 & 2.3\%  & Healthcare  & Numerical \\
        11 & Mushroom1   & 4429  & 22  & 221 & 5.0\%  & Botany      & Categorical\\
        12 & Mushroom2   & 4781  & 22  & 573 & 12.0\% & Botany      & Categorical\\
        13 & {Musk}        & 3062  & 166 & 97  & 3.2\%  & Chemistry   & Numerical \\
        14 & {Optdigits}   & 5216  & 64  & 150 & 2.9\%  & Image       & Numerical \\
        15 & PageBlocks  & 5393  & 10  & 510 & 9.5\%  & Document    & Numerical \\
        16 & Sick        & 3613  & 29  & 72  & 2.0\%  & Medical     & Mixed     \\
        17 & Thyroid     & 9172  & 28  & 74  & 0.8\%  & Medical     & Mixed     \\
        18 & Waveform    & 3443  & 21  & 100 & 2.9\%  & Physical    & Numerical \\
        19 & Wilt        & 4819  & 5   & 257 & 5.3\%  & Botany      & Numerical \\
        20 & {Yeast}       & 1484  & 8   & 507 & 34.2\% & Biology     & Numerical \\
        \bottomrule
        \end{tabular}}
        \end{table*}

\subsection{Experiment settings}
    Following previous works \cite{zhao2019PyOD, han2022ADbench, Jiang2023WSAD}, we adapt the unsupervised detection algorithms for predicting the new coming data, i.e., a fixed number (e.g., 5) of labeled outliers (unlabeled data are taken as negative instances) are used as validation set for tuning their hyper-parameters. For SSOD, we use various levels of supervision data {for training and the rest for testing. Specifically, the number of labeled outliers ranges from 5 to 30. {For COD, the number of candidate negative instances $N_-$ is set to 100.} Each experiment is repeated 10 times independently} and the average is reported.

\subsection{Comparison methods and settings}
    We select the following 10 baseline methods to evaluate the comparison performances of the detection algorithms, including 5 unsupervised and 5 semi-supervised ones. Most of them, except for WFRDA, are implemented by the AD library PyOD \cite{zhao2019PyOD, han2022ADbench} { and WSAD \cite{Jiang2023WSAD}}.

    \textbf{Unsupervised detection methods:}
    \begin{itemize}
    \item IForest (2008) \cite{liu2008isolation}: An ensemble model which isolates instances and constructs outlier scores using binary trees. The number of base estimators is tuned among \{5, 10, 50, 100, 500\}.
    \item SOD (2009) \cite{Kriegel2009SOD}: A subspace learning-based model that recognizes outliers by constructing subspaces of high-dimensional data. The parameter number of neighbors is found among \{2, 3, 5, 10, 20, 50\}.
    \item DeepSVDD (2018) \cite{Ruff2018DeepSVDD}: A deep learning-based model that learns embeddings of objects by minimizing the volume of a data hypersphere. The hyper-parameter number of epochs is selected in \{20, 50, 100, 200\}.
    \item ECOD (2022) \cite{li2022ecod}: A probabilistic model that predicts the empirical cumulative distribution of each attribute with the assumption that anomalies lie at the tails of the distribution. This model is parameter-free.
    \item WFRDA (2023) \cite{Yuan2023WFRDA}: A FRS-based algorithm that uses fuzzy-rough density and fuzzy entropy to describe the outlier degree of data. The parameter fuzzy radius is chosen from 0.1 to 2.0 with a stepsize of 0.1. 
    \end{itemize}

    \textbf{Semi-supervised approaches:}
    \begin{itemize}
    \item REPEN (2018) \cite{Pang2018RAMODO}: A deep learning-based model that uses a ranking approach to learn feature representations of ultra-high dimensional data. 
    \item DevNet (2019) \cite{Pang2019DevNet}: A rank-based method that introduces an end-to-end deep framework and incorporates a Gaussian prior and deviation loss to detect outliers. The confidence margin parameter $a$ is set to 5.
    \item DeepSAD (2020) \cite{Ruff2020DeepSAD}: A deep learning-based model that penalizes the inverse of the distances of outliers such that they are projected away from normal data. The balancing parameter $\eta$ in the objective function is fixed to 1.
    \item FEAWAD (2022) \cite{Zhou2022FEAWAD}: A rank-based method that utilizes a deep autoencoder to fit the normal data.
    \item PReNet (2023) \cite{Pang2023PReNet}: A rank-based method that scores the outliers by ordinal regression without involving any assumptions about outlier scores.
    \end{itemize}

\subsection{Evaluation metrics}
    We assess the detection methods by two popular metrics: AUC-ROC (Area Under Curve-Receiver Operating Characteristic) and AUC-PR (Area Under Curve-Precision Recall). They are both computed based on the outlier scores of objects. A larger AUC-ROC or AUC-PR value indicates better detection performance. {It is worth noting that AUC-ROC measures the overall ability of a model to discriminate between positive (outliers) and negative (inliers) instances and is less sensitive to class imbalance.
    AUC-PR, on the other hand, focuses on the performance of the model in capturing true outliers while minimizing false positives, and is more informative when dealing with imbalanced datasets with rare outliers. 
    The paired Wilcoxon signed-rank test \cite{Pang2023PReNet} is employed to assess the statistical significance of COD in comparison to its rival methods.}
    
\subsection{Experimental results and analysis}
\subsubsection{Comparative performance of detectors in real-world datasets}
    This part examines the detection performances in real-world datasets where there are a lot of unlabeled data and {a few labeled outliers are available}. Therefore, only {5 labeled outliers (other supervision levels are further examined in the following part) are taken to train the detectors, and all the unlabeled data are taken as the test set. The labeled outlier ratio of the experimental datasets ranges from 0.87\% to 11.9\%}. To ensure a fair comparison with unsupervised methods, we use the same training set with semi-supervised models to tune hyper-parameters, and the same testing set to evaluate the unsupervised algorithms.

    \begin{table*}[!ht]
    \centering
    \caption{AUC-ROC performances of comparison methods on 20 public datasets (@ 5 labeled outliers). The highest score is marked in bold.}\label{AUC}
    \resizebox{\textwidth}{!}{
    \begin{tabular}{c|ccccc|cccccc}
    \toprule
    Datasets    & IForest     & SOD            & DeepSVDD & ECOD           & WFRDA          & DeepSAD & REPEN     & DevNet         & FEAWAD      & PReNet         & COD                  \\
    \midrule
    Annealing   & 0.802          & 0.581          & 0.519 & 0.795          & 0.729          & 0.584       & 0.751       & {0.851}    & 0.832       & 0.847          & \textbf{0.857} \\
    {Annthyroid}  & 0.828          & 0.792          & 0.753 & 0.789          & 0.637          & 0.760       & 0.693       & {0.935}    & 0.794       & 0.931          & \textbf{0.986} \\
    Arrhythmia  & 0.798          & 0.731          & 0.619 & {0.807}    & 0.755          & 0.612       & 0.759       & 0.596          & 0.634       & 0.626          & \textbf{0.811} \\
    Audiology   & 0.774          & 0.563          & 0.552 & {0.837}    & 0.834          & 0.592       & 0.633       & 0.619          & 0.616       & 0.639          & \textbf{0.877} \\
    Breast      & \textbf{0.673} & 0.627          & 0.581 & {0.659}    & 0.657          & 0.570       & 0.649       & 0.507          & 0.528       & 0.493          & 0.649          \\
    {Breastw}     & 0.979          & 0.947          & 0.924 & {0.991}    & \textbf{0.992} & 0.838       & 0.987       & 0.764          & 0.827       & 0.704          & 0.989          \\
    {Cardio}      & {0.927}    & 0.634          & 0.522 & \textbf{0.935} & 0.914          & 0.733       & 0.889       & 0.775          & 0.795       & 0.776          & 0.871          \\
    CreditA     & {0.978}    & 0.844          & 0.865 & \textbf{0.991} & 0.975          & 0.739       & 0.948       & 0.805          & 0.709       & 0.813          & 0.963          \\
    Ionosphere  & 0.843          & \textbf{0.900} & 0.770 & 0.728          & 0.784          & 0.818       & {0.846} & 0.637          & 0.617       & 0.597          & 0.785          \\
    {Mammography} & 0.866          & 0.798          & 0.615 & \textbf{0.906} & 0.839          & 0.870       & 0.872       & 0.896          & 0.857       & {0.903}    & 0.855          \\
    {Mushroom1}   & 0.931          & \textbf{0.983} & 0.542 & 0.949          & {0.971}    & 0.942       & 0.929       & 0.902          & 0.910       & 0.886          & 0.964          \\
    {Mushroom2}   & 0.879          & 0.883          & 0.628 & 0.866          & 0.882          & {0.904} & 0.888       & 0.892          & 0.828       & 0.890          & \textbf{0.976} \\
    {Musk}        & 0.999          & 0.738          & 0.669 & 0.959          & {0.999}    & 0.989       & 0.999       & 0.859          & 0.929       & 0.973          & \textbf{1.000} \\
    {Optdigits}   & 0.739          & 0.492          & 0.667 & 0.606          & 0.942          & 0.869       & 0.609       & 0.988          & {0.994} & \textbf{0.999} & 0.989          \\
    PageBlocks  & {0.904}    & 0.658          & 0.708 & \textbf{0.914} & 0.868          & 0.888       & 0.902       & 0.800          & 0.730       & 0.776          & 0.889          \\
    Sick        & 0.801          & 0.698          & 0.535 & 0.844          & 0.837          & {0.859} & 0.748       & 0.837          & 0.859       & 0.809          & \textbf{0.902} \\
    Thyroid     & 0.662          & 0.588          & 0.534 & 0.579          & 0.516          & 0.663       & 0.648       & 0.730          & 0.723       & {0.735}    & \textbf{0.740} \\
    Waveform    & 0.684          & 0.634          & 0.628 & 0.601          & 0.699          & 0.733       & 0.668       & \textbf{0.859} & 0.821       & {0.854}    & 0.749          \\
    Wilt        & 0.489          & 0.586          & 0.460 & 0.394          & 0.331          & 0.696       & 0.341       & {0.909}    & 0.867       & \textbf{0.928} & 0.583          \\
    {Yeast}       & 0.430          & 0.449          & 0.491 & 0.445          & 0.395          & 0.474       & 0.383       & \textbf{0.602} & 0.573       & {0.600}    & 0.446          \\
    \midrule
    Average     & {0.799}    & 0.706          & 0.629 & 0.780          & 0.778          & 0.757       & 0.757       & 0.788          & 0.772       & 0.789          & \textbf{0.844} \\
    \midrule
    p-value     & 0.012          & 0.001          & 0.001 & 0.01           & 0.001          & 0.001       & 0.002       & 0.02           & 0.009       & 0.027          & -             \\
    \bottomrule
    \end{tabular}}
    \end{table*}

    \begin{table*}[!h]
    \centering
    \caption{AUC-PR performances of comparison methods on 20 public datasets (@ 5 labeled outliers). The highest score is marked in bold.}\label{PR}
    \resizebox{\textwidth}{!}{
    \begin{tabular}{c|ccccc|cccccc}
    \toprule
    Datasets    & IForest & SOD          & DeepSVDD& ECOD         & WFRDA          & DeepSAD& REPEN      & DevNet         & FEAWAD         & PReNet         & COD             \\
    \midrule
    Annealing   & 0.211       & 0.076          & 0.091 & 0.199          & 0.110       & 0.085          & 0.115 & {0.494}    & 0.482          & \textbf{0.506} & 0.269          \\
    {Annthyroid}  & 0.325       & 0.233          & 0.209 & 0.269          & 0.169       & 0.253          & 0.198 & {0.616}    & 0.370          & 0.592          & \textbf{0.784} \\
    Arrhythmia  & 0.436       & 0.314          & 0.233 & {0.448}    & 0.373       & 0.227          & 0.373 & 0.251          & 0.301          & 0.279          & \textbf{0.502} \\
    Audiology   & 0.572       & 0.335          & 0.292 & 0.649          & {0.706} & 0.341          & 0.400 & 0.542          & 0.473          & 0.515          & \textbf{0.816} \\
    Breast      & 0.446       & 0.385          & 0.378 & {0.467}    & 0.462       & 0.355          & 0.418 & 0.320          & 0.338          & 0.318          & \textbf{0.469} \\
    {Breastw}     & 0.945       & 0.861          & 0.805 & \textbf{0.984} & {0.979} & 0.777          & 0.967 & 0.791          & 0.839          & 0.735          & 0.969          \\
    {Cardio}      & 0.553       & 0.221          & 0.160 & 0.562          & 0.516       & 0.310          & 0.474 & 0.543          & \textbf{0.621} & 0.577          & {0.595}    \\
    CreditA     & 0.815       & 0.459          & 0.479 & \textbf{0.916} & {0.855} & 0.331          & 0.643 & 0.621          & 0.501          & 0.650          & 0.845          \\
    Ionosphere  & {0.780} & \textbf{0.896} & 0.589 & 0.638          & 0.655       & 0.764          & 0.777 & 0.611          & 0.585          & 0.599          & 0.669          \\
    {Mammography} & 0.215       & 0.116          & 0.056 & 0.432          & 0.094       & 0.314          & 0.177 & {0.507}    & 0.275          & \textbf{0.515} & 0.258          \\
    {Mushroom1}   & 0.432       & 0.675          & 0.179 & 0.483          & {0.899} & 0.675          & 0.411 & 0.835          & 0.852          & 0.830          & \textbf{0.905} \\
    {Mushroom2}   & 0.382       & 0.610          & 0.398 & 0.365          & 0.673       & 0.610          & 0.401 & {0.835}    & 0.778          & 0.828          & \textbf{0.929} \\
    {Musk}        & 0.976       & 0.123          & 0.307 & 0.504          & {0.981} & 0.893          & 0.972 & 0.802          & 0.902          & 0.945          & \textbf{1.000} \\
    {Optdigits}   & 0.056       & 0.026          & 0.053 & 0.033          & 0.366       & 0.206          & 0.036 & {0.968}    & 0.842          & \textbf{0.985} & 0.843          \\
    PageBlocks  & 0.493       & 0.298          & 0.339 & 0.517          & 0.373       & 0.555          & 0.538 & \textbf{0.580} & 0.522          & {0.568}    & 0.452          \\
    Sick        & 0.056       & 0.064          & 0.043 & 0.063          & 0.059       & 0.107          & 0.051 & {0.315}    & \textbf{0.355} & 0.302          & 0.267          \\
    Thyroid     & 0.015       & 0.013          & 0.011 & 0.009          & 0.008       & 0.018          & 0.011 & {0.051}    & 0.046          & \textbf{0.067} & 0.049          \\
    Waveform    & 0.051       & 0.054          & 0.070 & 0.038          & 0.047       & \textbf{0.235} & 0.053 & 0.191          & 0.185          & {0.210}    & 0.063          \\
    Wilt        & 0.047       & 0.065          & 0.045 & 0.041          & 0.036       & 0.106          & 0.036 & 0.339          & \textbf{0.442} & {0.393}    & 0.061          \\
    {Yeast}       & 0.319       & 0.307          & 0.328 & 0.331          & 0.324       & 0.319          & 0.287 & \textbf{0.431} & 0.409          & {0.427}    & 0.335          \\
    \midrule
    Average     & 0.406       & 0.307          & 0.253 & 0.397          & 0.434       & 0.374          & 0.367 & 0.532          & 0.506          & {0.542}    & \textbf{0.554} \\
    \midrule
    p-value     & 0.001       & 0.001          & 0.001 & 0.006          & 0.001       & 0.001          & 0.001 & 0.298          & 0.139          & 0.378          & -             \\
    \bottomrule
    \end{tabular}}
    \end{table*}
    
    The overall experimental results on 20 datasets are listed in Tables ~\ref{AUC} and ~\ref{PR} (@ 5 labeled outliers). %Figure \ref{boxplot} shows the boxplots of the average AUC-ROC of 15 datasets. 
    One can observe that COD achieves the best AUC-ROC on 8 datasets, including Annealing, {Annthyroid}, Arrhythmia, Audiology, {Mushroom2}, {Musk}, Sick and Thyroid. 
    {The improvements of COD over 10 comparison methods are statistically significant at the 95\% confidence level.
    In the case of AUC-PR, COD demonstrates superior performance on 7 datasets, namely {Annthyroid}, Arrhythmia, Audiology, Breast, {Mushroom1}, {Mushroom2}, and {Musk}. Notably, COD's enhancements over 7 out of 10 comparison methods (with the exceptions of DevNet, FEAWAD, and PReNet) are statistically significant at the 99\% confidence level. }
    The results confirm the effectiveness of COD in detecting outliers with very limited labeled data from a wide range of applications.
        
    In addition, COD does not perform as well as other semi-supervised methods on { Waveform or Wilt. For instance, DevNet beats COD on Waveform by 14.7\% on AUC-ROC, PReNet achieves 59.2\% higher AUC-ROC than COD on Wilt. Similar results can be found in terms of AUC-PR.} 
    The reason may be that COD's assumption of similarity-based outliers fails in this case where there is an unknown outlier type. 

    \begin{table*}[!h]
    \centering
    \caption{Average performance of detectors w.r.t. data type (@ 5 labeled outliers).}\label{tab:datatype}
    \resizebox{\textwidth}{!}{
    \begin{tabular}{cc|ccccc|cccccc}
    \toprule
    Metric                   & Data type & IForest     & SOD   & DeepSVDD & ECOD  & WFRDA       & DeepSAD & REPEN & DevNet      & FEAWAD & PReNet         & COD            \\
    \midrule
    \multirow{3}{*}{AUC-ROC} & Mixed       & {0.808} & 0.689 & 0.614    & 0.803 & 0.762       & 0.691   & 0.771 & 0.764       & 0.751  & 0.766          & \textbf{0.855} \\
                             & Categorical & 0.814       & 0.764 & 0.576    & 0.828 & {0.836} & 0.752   & 0.775 & 0.730       & 0.720  & 0.727          & \textbf{0.866} \\
                             & Numerical   & 0.790       & 0.693 & 0.655    & 0.752 & 0.764       & 0.788   & 0.745 & 0.820       & 0.800  & {0.822}    & \textbf{0.831} \\
    \midrule
    \multirow{3}{*}{AUC-PR}  & Mixed       & 0.307       & 0.185 & 0.172    & 0.327 & 0.281       & 0.154   & 0.239 & 0.346       & 0.337  & {0.361}    & \textbf{0.386} \\
                             & Categorical & 0.458       & 0.501 & 0.312    & 0.491 & {0.685} & 0.495   & 0.407 & 0.633       & 0.610  & 0.623          & \textbf{0.780} \\
                             & Numerical   & 0.433       & 0.291 & 0.269    & 0.395 & 0.413       & 0.430   & 0.411 & {0.580} & 0.545  & \textbf{0.595} & 0.548         \\
    \bottomrule
    \end{tabular}}
    \end{table*}
    
\subsubsection{Detection performances in handling {heterogeneous data}} 
    As real-world datasets often include heterogeneous attributes that take different types of values{, this} part investigates the model performances in the scenarios with the categorical data types. Table \ref{tab:datatype} summarizes {the average results of detectors on three types of data, i.e., mixed, categorical and numerical data. COD achieves the best average scores with both AUC-ROC and AUC-PR across 5 mixed datasets and 4 categorical datasets, and also performs the best on average of numerical datasets in AUC-ROC. For example, COD performs better than PReNet by 11.6\% on mixed datasets in AUC-ROC.} This is due to the COD's utilization of FRS to directly process nominal attributes without introducing {extra data assumptions}, which is pretty effective in outlier detection in mixed attribute data. {However, we also observed that COD's average AUC-PR score on numerical datasets is relatively lower than that of PReNet (8.6\%) or DevNet (5.8\%). This may be due to the different characteristics or distributions of outliers in some cases (e.g., Waveform and Wilt).}

\subsubsection{Model performances in various levels of label supervision}
    Since it is hard to obtain labeled outliers in most outlier detection tasks, this part aims to study the detection performances of the comparison methods with respect to various levels of label supervision, and investigate the improvements of the semi-supervised methods gained from partial labels compared with the best-unsupervised algorithm. {The number of labeled outliers for training is set to vary from 5 to 30, while other unlabeled data are treated as normal instances during training. This setting complies with the previous works \cite{Pang2018RAMODO, Pang2019DevNet, Pang2023PReNet} which is viewed as training the detectors with noise. However, we do not deliberately manipulate the value of outlier contamination, since it is more practical in real-world scenarios with various levels of outlier contamination.}
    
    \begin{figure*}[!h]
        \centering
        \includegraphics[width=\textwidth]{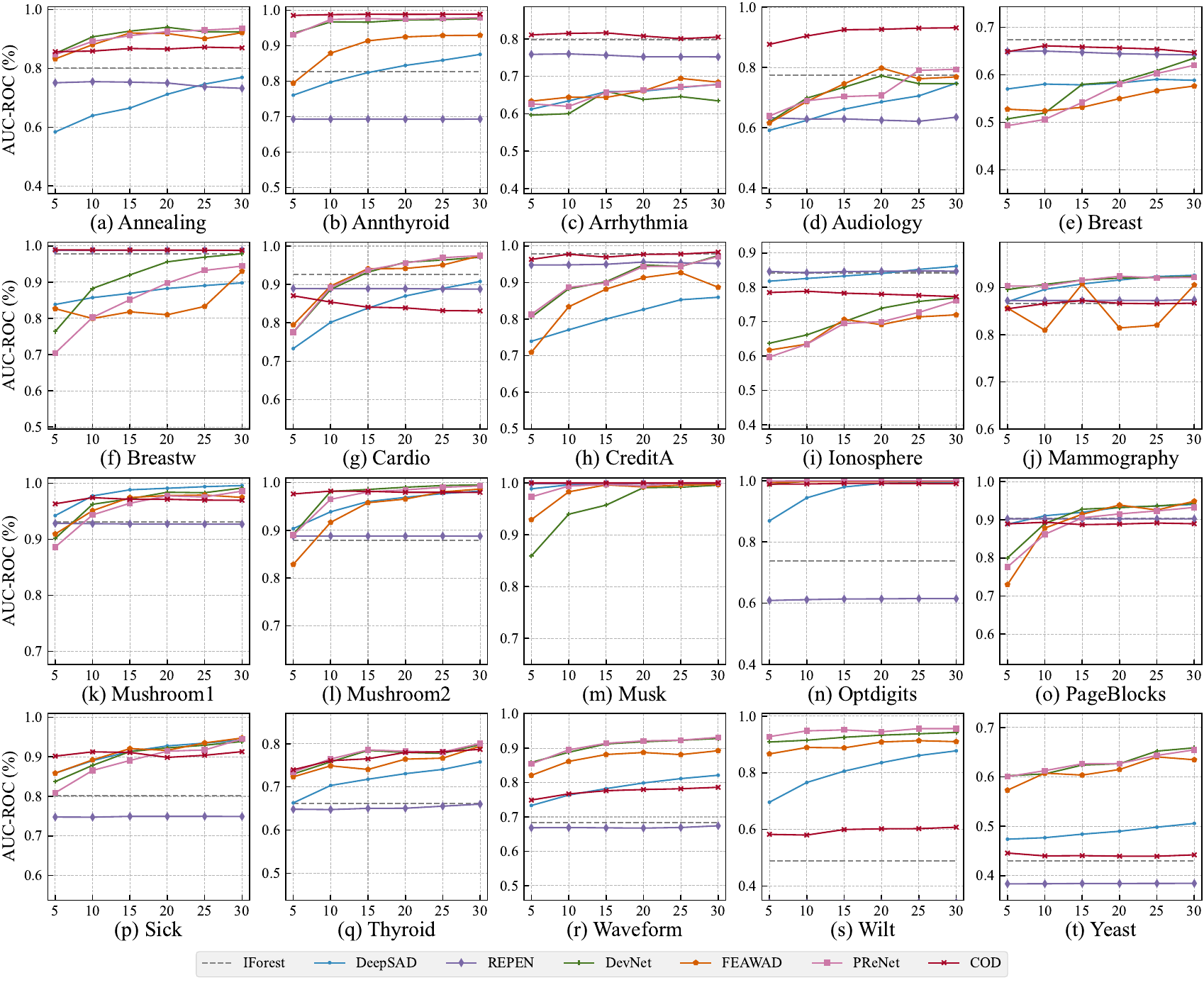}
        \caption{AUC-ROC scores w.r.t. multiple levels of supervision on 20 experimental datasets. The best-unsupervised method is IForest which achieves the highest average AUC-ROC score.}
        \label{fig:train_auc}
        \end{figure*}
        
    Figure \ref{fig:train_auc} depicts the AUC-ROC performances of the comparison methods w.r.t. multiple levels of supervision.%, and Figure \ref{fig:train_ave} shows the average on all datasets.
    These semi-supervised algorithms tend to improve with the increase of labeled {outliers}, as more labeled data provide more information for training. However, the AUC-ROC of some comparison deep learning-based models falls with more labeled data in some datasets, e.g., {FEAWAD and DevNet on Audiology, COD on {Cardio}.} This may be caused by the different distribution of outliers, and they may degrade detection performance when the added data have anomalous behavior that conflicts with other data. 
    However, we also noticed that COD does not improve significantly with an increasing level of supervision in some cases, e.g., {COD's improvements on Waveform, Wilt and {Yeast} are much lower than that of PReNet or FEAWAD. This may be due to the fact that COD leverages only a few unlabeled data (e.g., 100) to learn the outlier behavior, which makes it difficult to obtain a complete distribution of data in some scenarios, thus limiting detection performance.
        
    %Moreover, one can see, from Figure \ref{fig:train_ave}, {that }COD is the most data-efficient algorithm. It performs the best on an average of 15 datasets with the lowest number of labeled data. Impressively, COD needs only one-quarter of labeled data to achieve comparable results to the comparison methods DevNet, FEAWAD and PReNet.
    Furthermore, most semi-supervised detectors perform better than the best unsupervised method IForest when only a few labeled outliers are available. For instance, DeepSVD, DevNet, FEAWAD and achieve higher AUC-ROC scores than IForest on 13 out of 20 datasets at 30 labeled outliers, PReNet (COD) beats IForest on 15 (16) out of 20 datasets at the same level of supervision.}
    COD's advantage is due to its adoption of classification consistencies between attributes and decisions and its use to guide the outlier scoring. This allows COD to integrate unlabeled outlier factors and supervised guidance to construct outlier scores with the very limited labeled data.

    \begin{figure*}[!h]
        \centering
        \includegraphics[width=\textwidth]{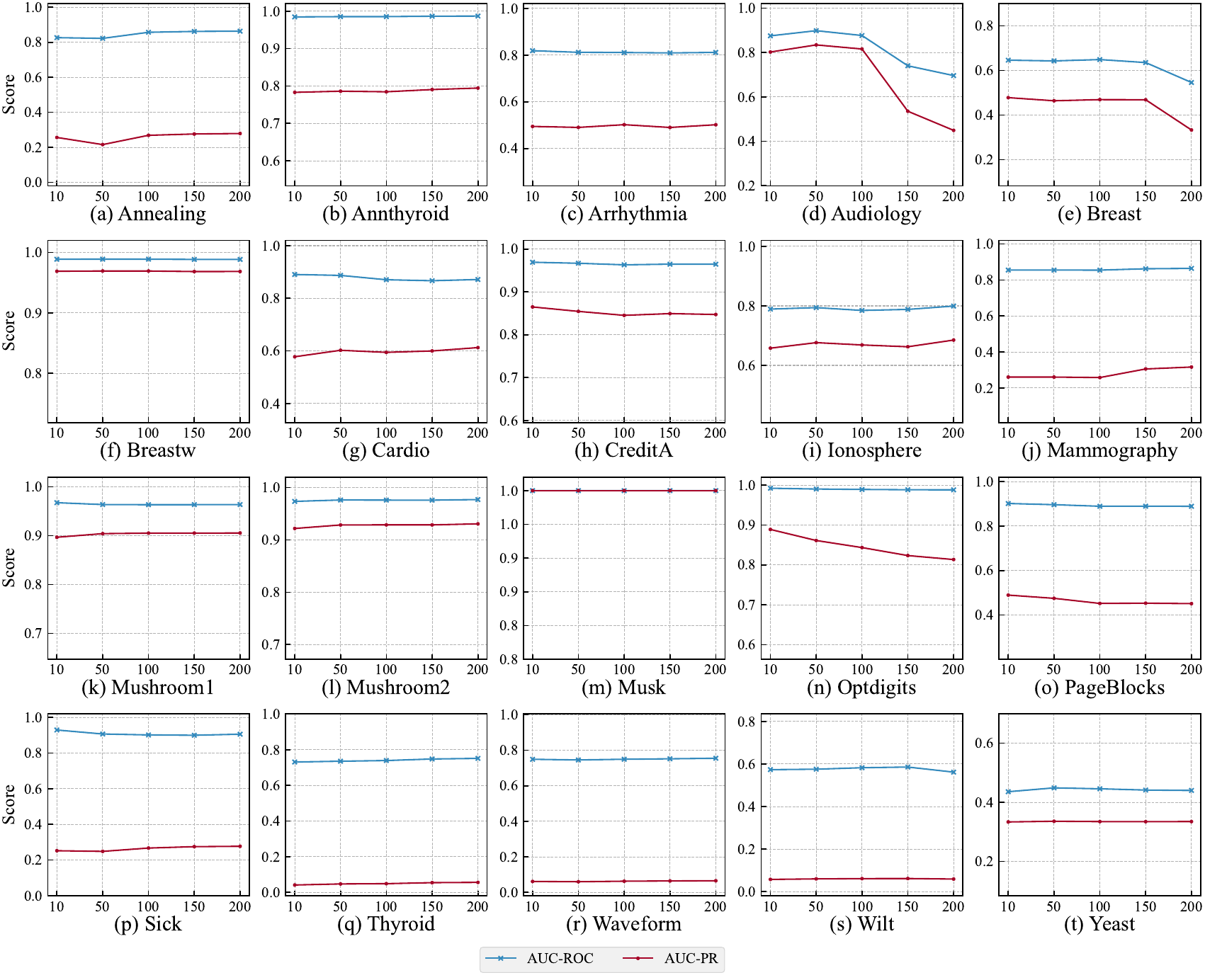}
        \caption{COD's performances w.r.t. the number of selected negative instances.}
        \label{fig:N_neg}
        \end{figure*}
        
\subsubsection{Detection performances w.r.t. the number of candidate negative instances}
    As COD introduces a heuristic method to collect normal objects, the number of candidate negative instances $N_-$ may have an influence on its detection performance. This part investigates the detection performances of COD w.r.t. the number of candidate negative instances. We tune $N_-$ among $\{10, 50, 100, 150, 200\}$ and report both AUC-ROC and AUC-PR at 5 labeled outliers. Figure \ref{fig:N_neg} depicts the results of COD. As the change of $N_-$, both COD's AUC-ROC and AUC-PR keep stable or slightly fluctuate in most datasets. 
    %However, one can observe that the two curves are not consistent in some cases, e.g., when the AUC-ROC improves on Waveform and Wilt, the AUC-PR presents a straight line. Conversely, the AUC-PR improves on Mushroom2, Optdigits and PageBlocks, the AUC-ROC stays the same.
    { Notably, the performance on the Musk dataset is distinct, with the two curves completely overlapping across all parameter settings. This occurs because COD achieves maximum scores for both metrics (i.e., 1) throughout all tested values of $N_-$. Therefore, COD} is robust to the choice of the number of selected negative instances, and it does not require a large number of negative instances to achieve good performance.

\section{Conclusion}
    {This paper proposes a novel consistency-guided outlier detection method (COD) for mixed data with the FRS theory.
    COD characterizes outliers with the fuzzy similarity class and introduces the classification consistency to guide the scoring of outliers, which improves the accuracy and efficiency of outlier detection with very limited labeled data.
    The label-informed fuzzy similarity relation designed in COD enables a downstream task-guided approach to determine the optimal fuzzy radius for the representation of heterogeneous data. The proposed algorithm may have the potential to advance the application of the FRS theory in real-world scenarios.
    Experimental results on various types of public datasets demonstrate the effectiveness of COD} in handling mixed attribute data.    
    Our future work will resort to identifying the types of outliers and further utilizing the implicit knowledge from available data to improve the model performance.    

\section{Acknowledgement}
This work was supported by the National Natural Science Foundation of China (62306196 and 62376230), Sichuan Science and Technology Planning Project (2023YFQ0020), and the Fundamental Research Funds for the Central Universities (YJ202245).

\bibliographystyle{model1-num-names} 
\bibliography{mybib}

\end{document}